\documentclass[11pt]{article}

\usepackage[preprint]{acl}

\usepackage{times}
\usepackage{latexsym}

\usepackage[T1]{fontenc}

\usepackage[utf8]{inputenc}

\usepackage{microtype}
\usepackage{float}
\usepackage{inconsolata}

\usepackage{booktabs}
\usepackage{graphicx}
\usepackage{algorithm}
\usepackage{algorithmic}
\usepackage[table,xcdraw]{xcolor}
\usepackage{multirow}
\usepackage{amsmath}
\usepackage{amssymb}
\usepackage{mathtools}
\usepackage{amsthm}
\newtheorem{theorem}{Theorem}[section]
\title{Resource‑Efficient Reinforcement for Reasoning Large Language Models via Dynamic One‑Shot Policy Refinement}

\author{Yunjian Zhang\thanks{Equal contribution.} \\
  UCAS \\
  \And
  Sudong Wang\footnotemark[1] \\
  HKUST(GZ) \\
  \And
  Yang Li \\
  Tsinghua University \\
  \And
  Peiran Xu \\
  Sun Yat-Sen University \\
  \AND
  Conghao Zhou \\
  Xidian University \\
  \And
  Xiaoyue Ma \\
  George Mason University \\
  \And
  Jianing Li \\
  Peking University \\
  \And
  Yao Zhu \thanks{Corresponding author.} \\
  Zhejiang University}

\begin{document}
\maketitle
\begin{abstract}
Large language models (LLMs) have exhibited remarkable performance on complex reasoning tasks, with reinforcement learning under verifiable rewards (RLVR) emerging as a principled framework for aligning model behavior with reasoning chains. Despite its promise, RLVR remains prohibitively resource-intensive, requiring extensive reward signals and incurring substantial rollout costs during training. In this work, we revisit the fundamental question of data and compute efficiency in RLVR. We first establish a theoretical lower bound on the sample complexity required to unlock reasoning capabilities, and empirically validate that strong performance can be achieved with a surprisingly small number of training instances. To tackle the computational burden, we propose Dynamic One-Shot Policy Refinement (DoPR), an uncertainty-aware RL strategy that dynamically selects a single informative training sample per batch for policy updates, guided by reward volatility and exploration-driven acquisition. DoPR reduces rollout overhead by nearly an order of magnitude while preserving competitive reasoning accuracy, offering a scalable and resource-efficient solution for LLM post-training. This approach offers a practical path toward more efficient and accessible RL-based training for reasoning-intensive LLM applications.
\end{abstract}

\section{Introduction}
Recent advances in large language models (LLMs) \cite{Shao2024DeepSeekMathPT,DeepSeekAI2025DeepSeekR1IR,Liu2025UnderstandingRT,Achiam2023GPT4TR,Hu2025OpenReasonerZeroAO} have led to significant progress in solving complex reasoning tasks. A key driver behind these advances is the adoption of reinforcement learning (RL) \cite{Zhang2025RightQI,Lin2025CPPOAT,Xiong2025AMA}, which has proven effective in activating coherent reasoning trajectories within LLMs. Leading LLMs, including GPT‑o1 \cite{Contributors2024OpenAIOS}, Gemini \cite{Reid2024Gemini1U}, and DeepSeek‑r1 \cite{DeepSeekAI2025DeepSeekR1IR}, all incorporate RL-based post-training stages to refine the models' reasoning abilities beyond what standard supervised fine-tuning (SFT) \cite{Chen2024SelfPlayFC,Tajwar2024Preference,Dong2023HowAI} can achieve. These developments underscore the growing importance of RL as a tool for enhancing the reasoning performance of LLMs.

Reinforcement learning with verifiable rewards (RLVR) \cite{Gao2024OnDE,Lambert2024TLU3P,Team2025KimiKS,Wang2025ReinforcementLF,Wang2025BeyondT8} has emerged as a prominent paradigm for training reasoning language models, which provides binary and verifiable feedback indicating the correctness of the model's answer. Recent work has primarily focused on improving the performance of RL algorithms \cite{Chen2025SEEDGRPOSE,Yue2025VAPOEA,Xu2025NotAR,Cui2025TheEM}, while comparatively little attention has been paid to their resource demands. From a data perspective, RLVR requires large quantities of samples with high-quality reward signals to drive effective learning. On the computational side, RLVR typically performs multiple rollouts over the entire dataset during training to ensure stable and effective gradient-based policy optimization, leading to substantial computational overhead. Although some work aims at improving training efficiency \cite{Wang2025ReinforcementLF,zhao2025ufo}, they rely on extensive pre-training over large datasets to identify the most suitable training instances, thereby failing to fundamentally reduce resource consumption. 

This work aims to reduce the data and computational overhead of RLVR during the post-training of reasoning LLMs. We begin by deriving a theoretical lower bound on the sample complexity required to elicit optimal reasoning behavior during RLVR. Surprisingly, our empirical findings reveal that even with a remarkably small number of training instances, the model can achieve near-optimal reasoning performance. This suggests that RL supervision serves more as a capability activator rather than a performance booster, and that the reasoning ability of the model is largely constrained by its pre-training rather than the volume of RL data. These results challenge the prevailing assumption that large-scale rewards are necessary, and instead highlight the potential of minimalist, data-efficient reinforcement strategies. These insights enable us to drastically reduce the training dataset to a compact subset without compromising final performance. In parallel, we introduce Dynamic One‑Shot Policy Refinement (DoPR), a novel reinforcement learning strategy that dynamically selects a single most informative instance within each mini-batch. By focusing policy updates on maximally impactful samples, DoPR significantly reduces the number of rollouts, yielding substantial savings in computational cost while maintaining comparable performance.

Our contributions can be summarized as follows:
\begin{itemize}
    \item We present the first systematic investigation into the nexus between data scale and reasoning performance within LLMs. By establishing a rigorous theoretical lower bound on sample complexity, we demonstrate that reasoning capabilities under RLVR can be activated by a minimalist data regime. This finding fundamentally decouples optimal reasoning performance from large-scale data requirements, challenging the long-standing paradigm that massive supervision is indispensable for effective RL alignment.
    \item We propose DoPR, a lightweight yet effective strategy that leverages historical rewards to select a single high-value training instance per batch for policy updates. This approach substantially reduces rollout overhead by nearly an order of magnitude while preserving reasoning performance, enabling efficient and cost-effective RL training for reasoning tasks.
    \item Extensive experiments reveal that our approach is capable of reducing the training dataset to as few as 16 instances and constraining the rollout budget to a single sample per batch, while maintaining competitive reasoning performance.
\end{itemize}

\section{Related Work}
\subsection{Reasoning Large Language Models}
Human-like reasoning has received growing attention for its potential to support generalization across abstract, multi-step tasks \cite{KojimaGRMI22,ZhouSHWS0SCBLC23}. Early efforts primarily focus on prompting methods to elicit latent reasoning capabilities from pre-trained LLMs. For example, Chain-of-Thought (CoT) prompting \cite{Wei0SBIXCLZ22} enables step-by-step reasoning without any additional training, while more structured paradigms such as Tree-of-Thought (ToT) \cite{YaoYZS00N23} and Graph-of-Thought (GoT) \cite{Besta2023GraphOT} incorporate hierarchical or graph-based reasoning paths. In addition, self-consistency decoding \cite{WSLCNCZ23} enhances reliability by aggregating multiple sampled reasoning traces.Some efforts have investigated more effective training strategies to improve the model’s intrinsic reasoning competence. For instance, LIMO \cite{Ye2025LIMOLI} applies SFT on curated mathematical reasoning datasets to explicitly guide the model toward correct solution steps. Moreover, some researchers utilize LLM-driven search algorithms to automatically generate accurate reasoning trajectories through trial-and-error search \cite{ZhengC00WZL0LXZ23}, and then train Process Reward Models (PRMs) on these reasoning trajectories \cite{Uesato2022SolvingMW}, which provide dense and intermediate rewards to enable reinforcement learning over reasoning chains.

\subsection{Reinforcement Learning with Verifiable Rewards}
RLVR has emerged as a principled framework for enhancing reasoning capabilities, particularly in domains where correctness can be objectively assessed, such as mathematical problem solving. Instead of relying on human feedback or learned reward models, RLVR employs rule-based verification mechanisms to generate reward signals, enabling robust optimization of reasoning policies with minimal annotation cost. The effectiveness of RLVR is first demonstrated by OpenAI’s o1 series \cite{Contributors2024OpenAIOS}, and subsequent models such as DeepSeek-R1 \cite{DeepSeekAI2025DeepSeekR1IR} and the Qwen series \cite{Bai2023QwenTR} further advance the use of verifiable feedback in training reasoning LLMs. On the algorithmic front, GRPO \cite{Shao2024DeepSeekMathPT,DeepSeekAI2025DeepSeekR1IR} and its extensions \cite{Liu2025UnderstandingRT,Zhang2025SRPOAC,kong2025rethinking} have become foundational in enabling efficient reward-driven training. Follow-up frameworks such as DAPO \cite{Yu2025DAPOAO}, VAPO \cite{Yue2025VAPOEA}, SimpleRLZoo \cite{Zeng2025SimpleRLZooIA}, and Open-Reasoner-Zero \cite{Hu2025OpenReasonerZeroAO} further explore various design choices for policy optimization, reward shaping, and data reuse under verifiable reward settings. 

Despite the notable successes, existing RLVR strategies suffer from substantial resource demands. First, they typically require a large amount of reward signals to supervise policy learning. Although some efforts \cite{zhao2025ufo,Wang2025ReinforcementLF} reduce the number of training samples, they still rely on preselection throughout the dataset and thus do not fundamentally reduce the effective sample size. Second, GRPO-based algorithms incur significant computational overhead during training, as they rely on repeated rollouts over the entire dataset to estimate gradients and improve policies. This work takes resource efficiency as the central goal, aiming to reduce the data and compute demands of RLVR without sacrificing its reasoning performance.

\begin{figure}[htb]
	\begin{minipage}[b]{.49\linewidth}
		\centering
		\centerline{\includegraphics[width=4.3cm]{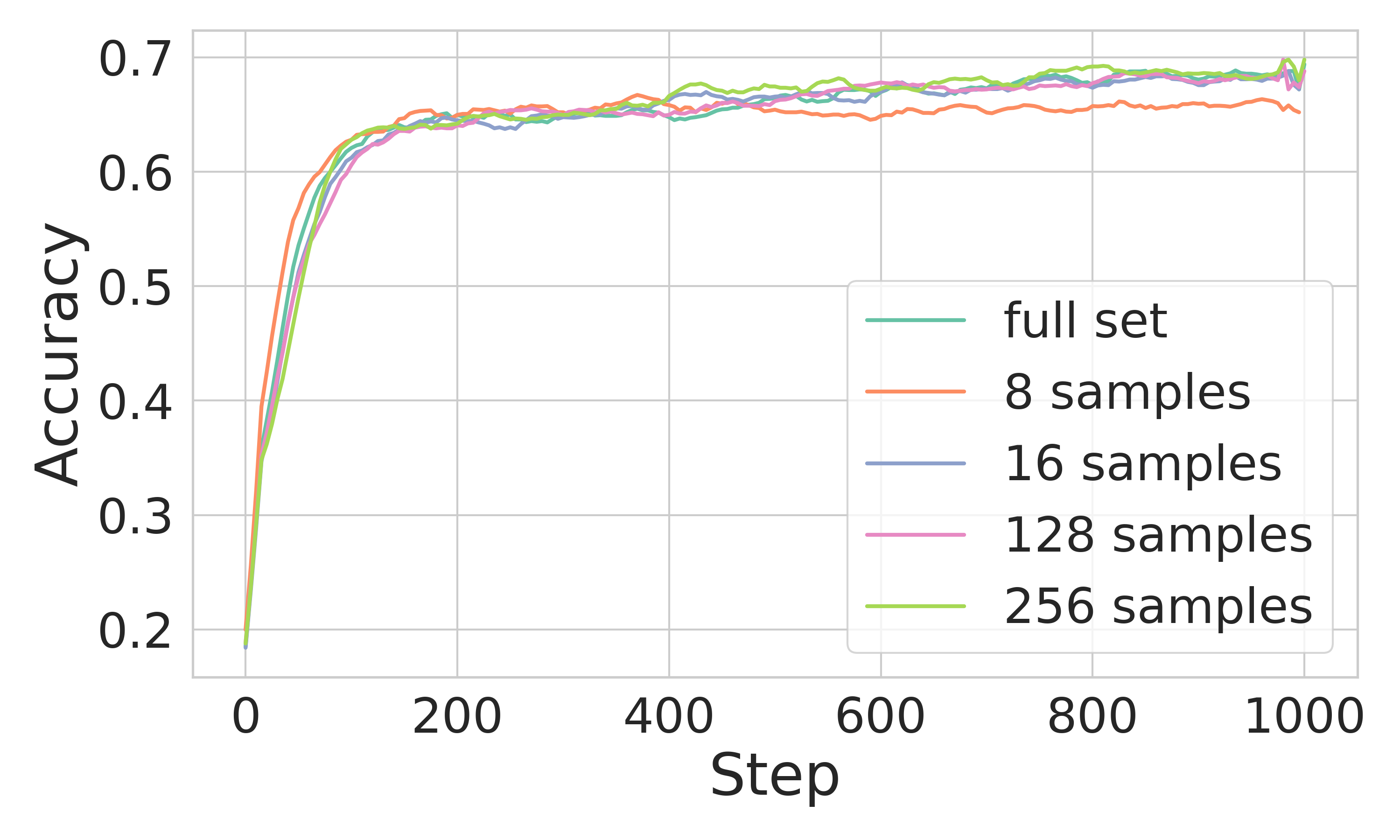}}
		\centerline{(a) Math}\medskip
	\end{minipage}
	\begin{minipage}[b]{.49\linewidth}
		\centering
		\centerline{\includegraphics[width=4.3cm]{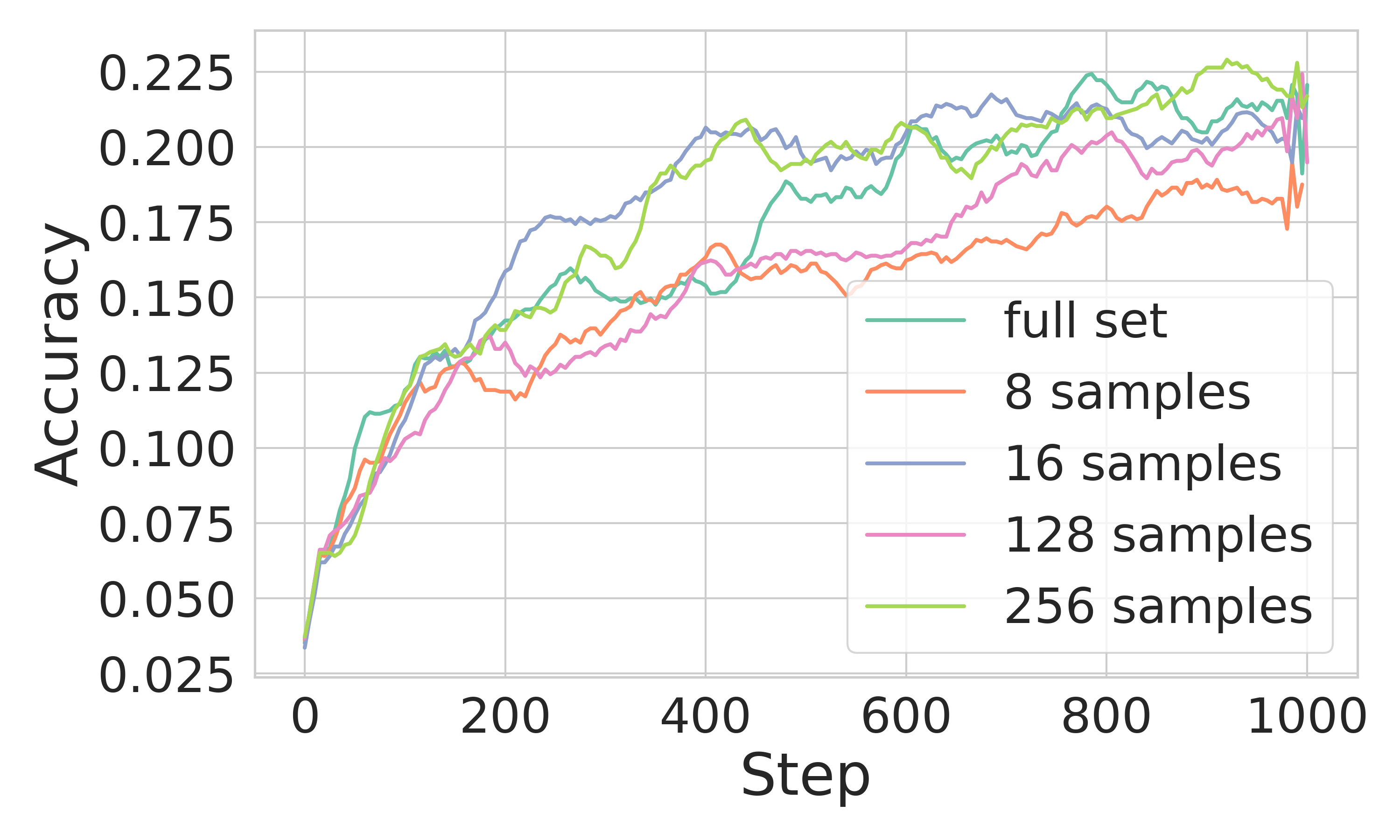}}
		\centerline{(b) Minierva-Math}\medskip
	\end{minipage}
	\caption{Accuracy on MATH and Minierva-MATH with varying training set sizes. Except for the 8-sample setting, all configurations converge to comparable performance, indicating that strong reasoning ability can be achieved with few training examples.}
	\label{fig:val_curve}
\end{figure}

\section{Rethinking Data Requirements for RLVR}
We begin by revisiting a fundamental question: \emph{How much data is truly required to unlock the reasoning capabilities of LLMs through RLVR?} Prior studies have explored various strategies for reducing the training data scale in RLVR \cite{Wang2025ReinforcementLF,Ye2025LIMOLI}. These methods typically rely on carefully crafted generation or selection strategies to identify a subset of high-quality samples for reinforcement learning. While such heuristics can be empirically effective, they provide only a pragmatic workaround rather than a principled understanding of the relationship between data volume and policy quality. Crucially, they do not address the core question of \emph{whether strong reasoning performance can be achieved with limited reward signals.} Beyond heuristic-based sample selection, we investigate whether a small number of training instances can provably yield non-trivial gains in reasoning performance under verifiable rewards. We ground our analysis in the policy gradient formulation used in RLVR, and derive a lower bound of required data scale on expected policy improvement.

For a pre-training policy $\pi_{\theta_0}$, its performance gap with the optimal policy $\pi_{\theta^{\star}}$ on a reasoning task is:
\begin{equation}
    \mathbb{E}_{x\sim\mathcal{D}}[P^{\pi_{\theta^{\star}}}(x)-P^{\pi_{\theta_0}}(x)]=\epsilon,
\end{equation}
where $x$ is an instance sampled from the dataset $\mathcal{D}$, and $P^{\pi}(x)$ denotes the expected verification reward.

\begin{theorem}
Consider an optimization procedure where the policy is updated using a single sample at each step, let $\pi_{\theta^\star}$ denotes the optimal policy and $\pi_{\theta_N}$ denotes the policy after $N$ updates. To guarantee that the expected performance gap satisfies
\begin{equation}
    \mathbb{E}_{x \sim \mathcal{D}} \left[ P^{\pi_{\theta^\star}}(x) - P^{\pi_{\theta_N}}(x) \right] \leq \epsilon',
\end{equation}
it suffices that the number of steps $N$ satisfies
\begin{equation}
    N \geq \mathcal{O}( \ln \frac{\epsilon}{\epsilon'}),
\end{equation}
where $\epsilon$ denotes the initial performance gap.
\end{theorem}

\begin{proof}
Let \( J(\theta) = \mathbb{E}_{x \sim \mathcal{D}} \left[ P^{\pi_\theta}(x) \right] \) denote the expected performance of the policy parameterized by \( \theta \). Assume that \( J \) is locally \( L \)-smooth \cite{Chen2025UnderstandingPA}, i.e., there exists a constant $L>0$ and a neighborhood $\mathcal{N}$ containing the initial policy $\theta_0$, such that the following Lipschitz inequality holds for all \( \theta, \theta' \in \mathcal{N} \):
\begin{equation}
    J(\theta') \geq J(\theta) + \langle \nabla J(\theta), \theta' - \theta \rangle - \frac{L}{2} \| \theta' - \theta \|^2.
\end{equation}
We remark that this assumption holds primarily in the vicinity of the optimization trajectory, and its validity is empirically ensured by conservative step sizes, KL regularization, and gradient clipping. By the policy gradient theorem, \( J \)'s gradient is given by:
\begin{equation}
    \nabla J(\theta) = \mathbb{E}_{x \sim \mathcal{D}} \left[ \nabla_{\theta} P^{\pi_\theta}(x) \right].
\end{equation}
Assume that the policy is updated at each step via gradient ascent with a fixed learning rate \( \alpha > 0 \):
\begin{equation}
    \theta_{t+1} = \theta_t + \alpha g_t,
\end{equation}
where \( g_t \) is an unbiased stochastic estimate of the true gradient, i.e., \( \mathbb{E}[g_t] = \nabla J(\theta_t) \), and the variance of \( g_t \) is bounded by $Var(g_t)\leq\delta^2$.
Taking expectations and applying the smoothness of \( J \), we obtain:
\begin{equation}
\small
    \mathbb{E}\left[J(\theta_{t+1}|\theta_t) \right]\geq J(\theta_t)+\alpha(1-\frac{L\alpha}{2})||\nabla J(\theta_t)||^2-\frac{L\alpha^2\delta^2}{2},
\end{equation}
Pretraining and supervised fine-tuning (SFT) endow the model with an initial reasoning capability. Coupled with the explicit gradient signal provided by RLVR, the objective \( J \) satisfies a local Polyak–Łojasiewicz (PL) condition in the neighborhood induced by pretraining and SFT \cite{Aich2025FromST,peng2024navigating}. Intuitively, the combination of pretraining/SFT (which shapes a favorable local manifold) and the supervised RLVR gradients renders the landscape locally well-conditioned for efficient optimization. That is, \( J \) satisfies
\begin{equation}
    \| \nabla J(\theta_t) \|^2 \geq c \left( J(\theta^\star) - J(\theta_t) \right) = c \Delta_t,
\end{equation}
for some constant \( c > 0 \), where \( \theta^\star \) is the optimal parameter and \( \Delta_t = J(\theta^\star) - J(\theta_t) \) denotes the performance gap at step \( t \). Then, we have:
\begin{equation}
\small
    \mathbb{E} \left[ J(\theta_{t+1}) - J(\theta_t) \mid \theta_t \right] \geq \alpha c (1 - \frac{L\alpha}{2}) \Delta_t - \frac{L\alpha^2 \delta^2}{2}.
\end{equation}
Let \( \eta = \alpha c / 2 \), and choose \( \alpha \) such that \( 1 - \frac{L\alpha}{2} \geq \frac{1}{2} \), we further require
\[
    \frac{L \alpha^2 \delta^2}{2} \leq \frac{\eta \epsilon'}{2},
\]
which leads to the recursive inequality:
\begin{equation}
    \mathbb{E} \left[ \Delta_{t+1} \mid \theta_t \right] \leq (1 - \eta) \Delta_t + \frac{\epsilon'}{2}.
\end{equation}
Unfolding this recurrence, we obtain:
\begin{equation}
    \mathbb{E} \left[ \Delta_t \right] \leq (1 - \eta)^t \Delta_0 + \frac{\epsilon'}{2},
\end{equation}
where \( \Delta_0 = J(\theta^\star) - J(\theta_0) = \epsilon \) is the initial gap. To ensure \( \mathbb{E} \left[ \Delta_t \right] \leq \epsilon' \), it suffices that:
\begin{equation}
    (1 - \eta)^t \epsilon \leq \frac{\epsilon'}{2}.
\end{equation}
Taking logarithms on both sides yields:
\begin{equation}
    t \geq \frac{\ln(\epsilon'/2) - \ln(\epsilon)}{\ln(1 - \eta)} = \mathcal{O}( \ln \frac{\epsilon}{\epsilon'}),
\end{equation}
For more detailed derivation, please refer to the appendix. \qedhere
\end{proof}

In practical applications, the initial policy (such as the Qwen-math series) is typically a version that has already undergone supervised fine-tuning, which inherently possesses a certain degree of reasoning capability. Consequently, the initial performance gap is typically small, implying that $\ln(\frac{\epsilon}{\epsilon'})$ remains modest. Therefore, convergence can be achieved within a constant number of training steps.

To assess the impact of training data size on RLVR performance, we conduct a series of experiments using the Qwen2.5-Math-1.5B model \cite{yang2024qwen25mathtechnicalreportmathematical} on the MATH \cite{Hendrycks2021MeasuringMP} and Minerva-MATH \cite{Aitor2022Solving} benchmarks, varying the number of training samples across five settings: 8, 16, 128, 256, and the full dataset (1209 examples). As shown in Figure \ref{fig:val_curve}, we observe that all configurations with 16 or more samples converge to nearly identical validation accuracy. This result corroborates our theoretical analyses, and provides empirical evidence that only a small subset of the data is sufficient to elicit strong reasoning capabilities in RLVR training. Our analyses reveal that substantial reductions in training data are possible without sacrificing final performance, paving the way for more efficient RL-based reasoning.

\begin{algorithm}[t]
	\caption{Dynamic One-Shot Policy Refinement (DoPR)}
	\label{alg:dopr}
	\begin{algorithmic}[1]
		\REQUIRE Initial policy $\pi_{\theta_0}$, reference policy $\pi_{\text{ref}}$, rollout budget $G$, batch size $K$, momentum factors $\rho_1$, $\rho_2$, exploration weight $\lambda$, total steps $T$
		\STATE Initialize reward statistics $\mu_i^0 \leftarrow 0$, $\sigma_i^{0,2} \leftarrow 0$, and selection counter $n_i \leftarrow 0$ for each sample $x_i$ in the dataset $\mathcal{D}$
		\FOR{$t = 1$ to $T$}
		\STATE Sample a mini-batch $\mathcal{B} = \{o_1, \dots, o_K\}$ from the dataset
		\FORALL{$x_i \in \mathcal{B}$}
		\STATE Perform one rollout $o_i^t \sim \pi_{\theta_t}(x_i)$
		\STATE Compute scalar reward $r_i^t$
		\STATE Update exponential moving average of mean and variance:
		\STATE \hspace{1em} $\mu_i^t \leftarrow \rho_1 \cdot r_i^t + (1 - \rho_1) \cdot \mu_i^{t-1}$
		\STATE \hspace{1em} $\sigma_i^{t,2} \leftarrow \rho_2 \cdot (r_i^t - \mu_i^t)^2 + (1 - \rho_2) \cdot \sigma_i^{t-1,2}$
		\STATE Compute EM-UCB exploration term: $U_i^t \leftarrow \text{Sigmoid}(\frac{H_i^t-\mu_H}{\delta_H+\epsilon})\cdot\sqrt{\frac{\log(t+1)}{n_i+1}}$
		\STATE Compute acquisition score: $S_i^t \leftarrow \sigma_i^t + U_i^t$
		\ENDFOR
		\STATE Select top-scoring instance: $o^* \leftarrow \arg\max_{o_i \in \mathcal{B}} S_i^t$
		\STATE Increment selection count: $n_{o^*} \leftarrow n_{o^*} + 1$
		\STATE Perform $G$ rollouts on $o^*$: $\{o_j\}_{j=1}^G \sim \pi_{\theta_t}(o^*)$
		\STATE Compute group rewards and token-wise advantages $\hat{A}_{j,t}$ for GRPO
		\STATE Compute policy gradient loss $\mathcal{L}_{\text{GRPO}}(\theta)$
		\STATE Update policy $\theta_{t+1} \leftarrow \theta_t - \eta \nabla_{\theta} \mathcal{L}_{\text{GRPO}}(\theta)$
		\ENDFOR 
        \RETURN Trained policy $\pi_{\theta_T}$
	\end{algorithmic}
\end{algorithm}

\begin{figure*}[ht]
	\centering
	\includegraphics[width=.8\textwidth]{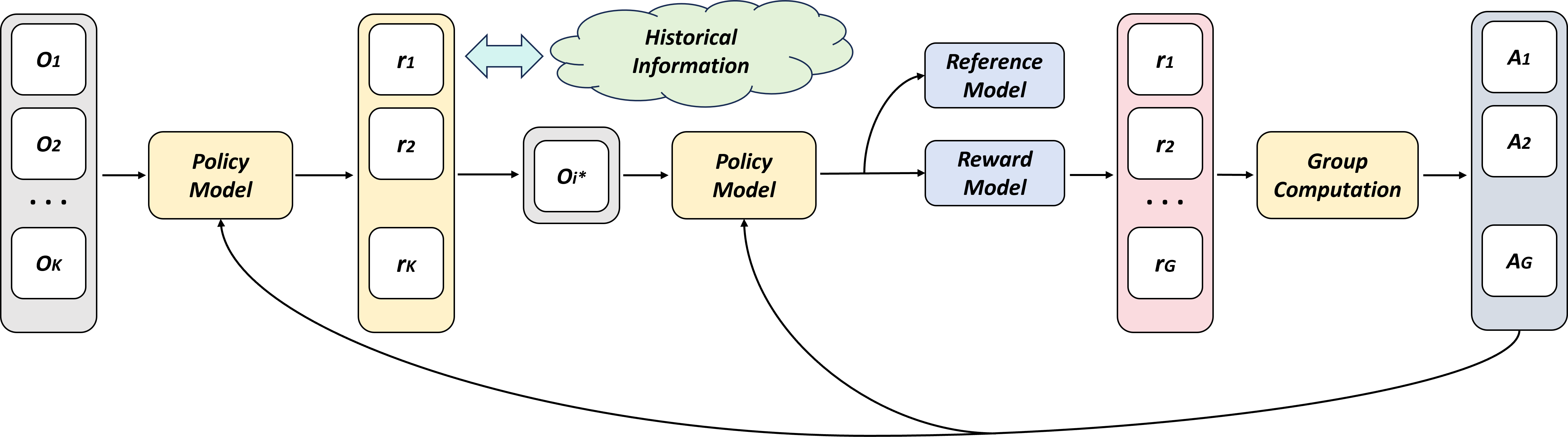}
	\caption{Overview of DoPR, which dynamically selects a single high-value training instance from each mini-batch based on historical reward statistics.}
	\label{fig:dopr}
\end{figure*}

\section{Dynamic One-Shot Policy Refinement}
Currently, reasoning model training often builds upon GRPO and its variants, which leverage group-wise relative rewards to estimate the advantages over the baseline. In particular, GRPO circumvents value function approximation by computing the average reward over multiple responses to the same input query, using this group-level statistic as a baseline for advantage estimation. For each question $q$, GRPO samples a group of $G$ responses $\{o_i\}_{i=1}^G$ from the current policy $\pi_{\theta_{old}}$, and optimizes the policy by maximizing the expected advantage over these responses:
{\small
\begin{equation}
\small
\begin{split}
    &\mathcal{J}_{\text{GRPO}}(\theta) = \ 
    \mathbb{E}_{q \sim P(Q), \{o_i\}_{i=1}^G \sim \pi_{\theta_{\text{old}}}} \bigg[
    \frac{1}{G} \sum_{i=1}^G \frac{1}{|o_i|} \sum_{t=1}^{|o_i|} \big\{ \\
    & \quad \min\left[ \mathcal{F}(\pi) \, \hat{A}_{i,t}, \; \operatorname{clip}_{1 - \epsilon}^{1 + \epsilon}(\mathcal{F}(\pi)) \hat{A}_{i,t} \right]-\beta \, \mathbb{D}_{\text{KL}}(\pi_\theta \,\|\, \pi_{\text{ref}}) \big\} \bigg],
\end{split}
\label{grpo}
\end{equation}
}
\begin{equation}
    \mathcal{F}(\pi)=\frac{\pi_\theta(o_{i,t} | q, o_{i,<t})}{\pi_{\theta_{old}}(o_{i,t} | q, o_{i,<t})},
\end{equation}
\begin{equation}
    \mathbb{D}_{\text{KL}}(\pi_\theta \,\|\, \pi_{\text{ref}})=\mathcal{F}(\pi)-\log \mathcal{F}(\pi)-1,
\end{equation}
where $\hat{A}_{i,t}$ is computed based on the group reward scores of the $t$-th token.

For each query in the batch, GRPO performs $G$ rollouts to estimate group-wise rewards, which are subsequently used to compute response-level advantages. Given a batch size of $K$, each policy update necessitates $K\times G$ rollouts, resulting in substantial computational overhead. To address this issue, we introduce Dynamic One-Shot Policy Refinement (DoPR), a lightweight and computationally efficient strategy that selectively refines the policy using only a single high-influence training instance, as shown in Figure \ref{fig:dopr}.

For each sample $o_i$ in a mini-batch $\{o_i\}_{i=1}^K$, we maintain a historical record of the reward statistics obtained from previous rollouts, including the mean $\mu_i$ and variance $\sigma_i^2$. In the $t$-th step, we execute a single rollout for $o_i$ using the current policy $\pi_{\theta_t}$, yielding a reward $r_i^{t}$, then calculate the exponentially weighted momentum estimates of the reward:
\begin{equation}
    \begin{aligned}
        \mu_i^{t} &= \rho_1 \cdot r_i^{t} + (1 - \rho_1) \cdot \mu_i^{t-1}, \\
        \sigma_i^{t,2} &= \rho_2 \cdot (r_i^{t} - \mu_i^{t})^2 + (1 - \rho_2) \cdot \sigma_i^{t-1,2},
    \end{aligned}
\end{equation}
where $\rho_1$ and $\rho_2$ are momentum hyperparameters controlling the temporal sensitivity. These momentum statistics allow us to capture the reward volatility of each sample over time, providing a principled measure of uncertainty under the current policy.

To address the exploration–exploitation trade-off, we propose an Entropy-Modulated Upper Confidence Bound (EM-UCB) acquisition score. Standard UCB \cite{kaelbling1994associative,garivier2011upper} treats all samples uniformly, inducing unnecessary exploration even when the policy has already converged on stable predictions. EM-UCB remedy this limitation by incorporating the policy entropy as a confidence-aware gate. As a result, exploratory updates are emphasized only for samples where the policy remains uncertain, aligning exploration with potential information gain, leading to more efficient rollout utilization without introducing additional training complexity.


\begin{equation}
    U_i^t = \text{Sigmoid}(\frac{H_i^t-\mu_H}{\delta_H+\epsilon})\cdot\sqrt{\frac{\log(t+1)}{n_i+1}},
\end{equation}
\begin{equation}
    H_i^t=-\frac{1}{\mathcal{T}}\sum^\mathcal{T}_{l=1}\sum_{v\in\mathcal{V}}\pi_\theta(v|o_i,<l)\log\pi_\theta(v|o_i,<l),
\end{equation}
where $n_i$ denotes the cumulative number of times sample $o_i$ has been selected, $\mu_H$, $\delta_H$ are the running mean and standard deviation of entropy values across the batch, $\mathcal{T}$ denotes the output length, and $\mathcal{V}$ is the vocabulary. The resulting EM-UCB term adaptively regulates exploration: when the policy exhibits high uncertainty (large $H_i^t$), the exploration bonus is amplified; when the policy is confident, the UCB contribution is suppressed, preventing unnecessary exploration on already-determined samples. Therefore, the composite acquisition score for each sample is defined as:
\begin{equation}
    S_i^{t} = \sigma_i^{t} + U_i^t.
\end{equation}
\begin{equation}
    i^{t} = \arg\max_{i} S_i^{t}.
\end{equation}
This scoring function balances exploitation (selecting samples with high reward variance) and exploration (favoring underused high-entropy samples). The core behind the exploitation is that the model is more uncertainty about the samples with high reward variance and high entropy, which are likely to be more informative for policy refinement.

After selecting the most informative sample $o_{i^{t}}$, we allocate the full rollout budget $G$ on it, collecting $G$ responses $\{o_{i^{t},j}\}_{j=1}^G$ from the current policy $\pi_{\theta_t}$. The policy is then updated using the group-wise relative rewards of these responses, following the GRPO formulation in Eq.(\ref{grpo}).

In contrast to the conventional GRPO that requires $K \times G$ rollouts per step, DoPR reduces the rollout cost to $G+(K-1)$, where $G$ rollouts are performed only on the selected sample and $(K-1)$ rollouts are performed on the remaining. This results in nearly an order-of-magnitude reduction in rollout cost, while still maintaining effective policy learning through targeted updates. 
DoPR’s design is simple, general, and easily integrable into existing RLVR pipelines. The pseudo code of DoPR is presented in Algorithm \ref{alg:dopr}.

\begin{table*}
	\caption{Pass@1 performance comparison across multiple mathematical reasoning benchmarks.}
	\label{tab:comparison-results}
	\centering
	\begin{tabular}{lccccccc}
		\toprule
		\rowcolor[HTML]{f8f9fa}
		\textbf{Method (Total Rollouts Fixed)} & \textbf{AIME24} & \textbf{AMC} & \textbf{MATH} & \textbf{MIN.} & \textbf{OLY.} & \textbf{GSM8K} & \textbf{Avg.} \\
		\midrule
		\rowcolor[HTML]{f8f9fa}
		\multicolumn{8}{l}{\textit{Baseline methods}} \\
		Qwen2.5-base \textbf{1.5B} & 0.0 & 0.0 & 3.2 & 2.2 & 2.4 & 3.9 & 1.9 \\
        Qwen2.5-base \textbf{7B} & 13.3 & 45 & 53.4 & 15.1 & 27.4 & 59.1 & 35.6 \\
		Qwen2.5-Math-base \textbf{1.5B} & 3.3 & 25.0 & 22.6 & 6.2 & 13.3 & 23.5 & 15.0 \\
        Qwen2.5-Math-base \textbf{7B} & 10 & 42.5 & 44.6 & 12.1 & 14.4 & 53.4 & 29.5 \\
		GRPO Qwen \textbf{1.5B} & 20.0 & 60.0 & 71.8 & 29.8 & 34.2 & 82.0 & 49.6 \\
        GRPO Qwen \textbf{7B} & 26.7 & 65 & 78.2 & 29.4 & 37.9 & 87.6 & 54.1 \\
        GRPO LLaMA \textbf{8B} & 3.3 & 10.0 & 25.8 & 11.4 & 6.5 & 70.9 & 21.3\\
		One-Shot RL Qwen \textbf{1.5B} & 13.3 & 55.2 & 66.6 & 19.5 & 29.8 & 76.6 & 43.5 \\
        One-Shot RL Qwen \textbf{7B} & 16.7 & 52.5 & 74.0 & 23.9 & 32.1 & 81.8 & 46.8 \\
        One-Shot RL LLaMA \textbf{8B} & 0.0 & 15.0 & 28.0& 16.9 & 7.7 & 72.8 & 17.6 \\
        UFO Qwen \textbf{1.5B} & 19.4 & 59.2 & 73.3 & 29.7 & 33.5 & 81.7 & 49.4 \\
        UFO Qwen \textbf{7B} & 27.9 & 52.8 & 73.5 & 38.7 & 38.2 & 85.3 & 52.7 \\
        UFO LLaMA \textbf{8B} & 3.1 & 14.4 & 26.8 & 16.4 & 8.1 & 72.6 & 23.6 \\
		\rowcolor[HTML]{f8f9fa}
		\multicolumn{8}{l}{ \textit{Our method}} \\
		DoPR Qwen \textbf{1.5B} & 19.7 & 59.3 & 73.5 & 29.4 & 33.9 & 82.2 & 49.6 \\
        DoPR Qwen \textbf{7B} & 30.0 & 52.5 & 73.8 & 39.0 & 37.9 & 85.6 & 53.1 \\
        DoPR LLaMA \textbf{8B} & 3.3 & 15.0 & 28.0 & 16.9& 7.7 & 72.8 & 24.0 \\
		\bottomrule
		\end{tabular}
\end{table*}

\section{Experiments}
\subsection{Experimental Settings}
We conduct experiments using the Qwen2.5-Math and LLaMA3.1 \cite{yang2024qwen25mathtechnicalreportmathematical} as the base language models, which are widely used for mathematical reasoning tasks. For training, we adopt the same setting with \cite{Wang2025ReinforcementLF}, which randomly selects a subset of 1209 high-quality reasoning examples from the DeepScaleR-Preview-Dataset \cite{deepscaler2025} to construct the training set, and use GRPO as the policy optimization backbone under consistent rollout configurations. The evaluation is performed on seven widely used mathematical reasoning benchmarks, including AIME24, AMC, MATH \cite{Hendrycks2021MeasuringMP}, Minerva Math \cite{Aitor2022Solving}, OlympiadBench \cite{Huang2024OlympicArenaBM}, and GSM8K \cite{cobbe2021gsm8k}. The evaluation metric is Pass@1 accuracy, which measures the percentage of correct answers in a single attempt. We choose UFO RL \cite{zhao2025ufo}, One-Shot RL \cite{Wang2025ReinforcementLF} without entropy strategy, and standard GRPO \cite{Shao2024DeepSeekMathPT} as the baselines, and the number of total training steps for all methods is set to 1000. From Figure \ref{fig:val_curve}, we observe that the sample size of 16 is sufficient to attain competitive performance, thus we adopt the 16-sample configuration for DoPR in experiments.

\subsection{Benchmarking Across Diverse Tasks}
Table~\ref{tab:comparison-results} presents the pass@1 performance of various reinforcement learning methods. We observe that RLVR brings substantial improvements over the base model. For instance on the Qwen2.5 1.5B model, GRPO outperforms the base model by a large margin across all datasets, improving the average accuracy from 37.8 to 49.6. This demonstrates the effectiveness of outcome-based policy optimization in enhancing reasoning capabilities. Our proposed method, DoPR, achieves comparable performance to GRPO with an average accuracy of 49.6, while requiring significantly fewer samples and rollouts during training, confirming its efficiency without sacrificing effectiveness. We also observe that One-Shot RL underperforms both GRPO and DoPR across most benchmarks. This performance gap can be attributed to its reliance on single-response updates, which provide limited gradient information and may lead to unstable training dynamics. In our experiments, One-Shot RL often exhibits early saturation and struggles to make further progress in later training stages due to gradient degradation. In contrast to uniform or random sampling approaches, DoPR dynamically identifies and exploits informative training instances to refine the policy, resulting in consistently stable learning behavior and improved sample efficiency throughout training.

\begin{table*}
	\caption{Pass@1 accuracy under varying sample budgets in RLVR training.}
	\label{tab:number-results}
	\centering
	\begin{tabular}{lccccccc}
		\toprule
		\rowcolor[HTML]{f8f9fa}
		\textbf{Method} & \textbf{AIME24} & \textbf{AMC} & \textbf{MATH} & \textbf{MIN.} & \textbf{OLY.} & \textbf{GSM8K} & \textbf{Avg.} \\
		\midrule
		\rowcolor[HTML]{f8f9fa}
		\multicolumn{8}{l}{\textit{Baseline methods}} \\
		GRPO (Full data) & 20.0 & 60.0 & 71.8 & 29.8 & 34.2 & 82.0 & 49.6 \\
		GRPO (256 samples) & 19.7 & 57.7 & 70.8 & 28.7 & 34.1 & 82.6 & 48.9 \\
		GRPO (128 samples) & 20.0 & 57.5 & 71.7 & 29.8 & 32.1 & 82.7 & 48.9 \\
		GRPO (16 samples) & 19.3 & 58.5 & 69.8 & 28.2 & 32.9 & 81.5 & 48.4 \\
		GRPO (8 samples) & 10.0 & 50.0 & 69.8 & 25.9 & 31.3 & 77.0 & 44.0 \\
		\rowcolor[HTML]{f8f9fa}
		\multicolumn{8}{l}{ \textit{Our methods}} \\
		DoPR (128 samples) & 19.9 & 60.1 & 73.0 & 27.8 & 32.7 & 82.1 & 49.3 \\
		DoPR (16 samples) & 19.7 & 59.3 & 73.5 & 29.4 & 33.9 & 82.2 & 49.6 \\
		DoPR (8 samples) & 9.3 & 44.7 & 68.9 & 26.4 & 30.8 & 77.4 & 42.9 \\
		\bottomrule
	\end{tabular}
\end{table*}

\subsection{Impact of Training Data Scale}
We compare the performance of GRPO and DoPR with different scale of training data on Qwen 1.5B, and the results are shown in Table~\ref{tab:number-results}. It can be observed that both GRPO and DoPR demonstrate remarkable robustness to data reduction. Specifically, performance remains largely stable when using 16 or more samples, with minimal degradation compared to training on the full dataset. For example, GRPO trained with only 16 samples achieves 69.8 on MATH and 81.5 on GSM8K, closely matching the full-data performance of 71.8 and 82.0, respectively. A comparable pattern is observed with DoPR, which also delivers strong results under the same data constraints. This suggests that effective policy learning via RLVR does not necessarily require extensive training corpora. However, when the training set is reduced to only 8 samples, a consistent drop in performance is observed across all benchmarks. This suggests a critical threshold below which the diversity and quantity of training data are insufficient to elicit the model’s full reasoning potential. Importantly, the performance trend remains consistent across both optimization strategies, implying that the underlying data requirement characteristics are shared across GRPO and DoPR. These results validate our theoretical insights regarding the minimal data requirement for RLVR, and further demonstrate that DoPR can retain high performance even in low-resource regimes, offering a practical solution for data-efficient reasoning model training.

\begin{table*}[t]
	\caption{Pass@1 accuracy under equal total rollout budget across different reinforcement learning strategies.}
	\label{tab:rollout-efficiency}
	\centering
	\begin{tabular}{lccccccc}
		\toprule
		\rowcolor[HTML]{f8f9fa}
		\textbf{Method (Total Rollouts Fixed)} & \textbf{AIME24} & \textbf{AMC} & \textbf{MATH} & \textbf{MIN.} & \textbf{OLY.} & \textbf{GSM8K} & \textbf{Avg.} \\
		\midrule
		\rowcolor[HTML]{f8f9fa}
		\multicolumn{8}{l}{\textit{Rollout budget: 5k}} \\
		GRPO         & 6.7 & 37.5 & 40.6 & 10.3 & 23.7 & 47.8 & 27.8 \\
		One-Shot RL  & 8.5 & 41.2 & 62.9 & 15.3 & 26.6 & 70.5 & 37.5 \\
		DoPR        & 14.6 & 46.0 & 66.4 & 20.9 & 30.1 & 75.2 & 42.2 \\
		\rowcolor[HTML]{f8f9fa}
		\multicolumn{8}{l}{\textit{Rollout budget: 10k}} \\
		GRPO         & 10.0 & 50.0 & 52.0 & 12.1 & 25.6 & 61.5 & 35.2 \\
		One-Shot RL  & 9.8 & 44.7 & 67.2 & 18.3 & 28.6 & 75.5 & 40.7 \\
		DoPR        & 18.2 & 57.1 & 67.5 & 26.3 & 30.7 & 77.1 & 46.1 \\
		\rowcolor[HTML]{f8f9fa}
		\multicolumn{8}{l}{\textit{Rollout budget: 30k}} \\
		GRPO         & 12.3 & 50.0 & 66.1 & 18.0 & 28.3 & 76.1 & 41.8 \\
		One-Shot RL  & 10.6 & 44.3 & 68.7 & 18.9 & 29.5 & 76.7 & 41.5 \\
		DoPR        & 18.6 & 57.9 & 74.3 & 30.2 & 33.7 & 80.1 & 49.1 \\
		\rowcolor[HTML]{f8f9fa}
		\multicolumn{8}{l}{\textit{Rollout budget: 250k}} \\
		GRPO         & 19.7 & 59.9 & 72.3 & 29.5 & 33.7 & 82.4 & 49.6 \\
		One-Shot RL & 13.8 & 54.2 & 66.1 & 19.1 & 30.8 & 77.6 & 43.6 \\
		DoPR        & 19.7 & 58.3 & 73.5 & 29.2 & 33.1 & 83.1 & 49.4 \\
		\bottomrule
	\end{tabular}
\end{table*}

\subsection{Comparison under Equal Rollout Budgets}
To evaluate the rollout efficiency of different reinforcement learning strategies, we conduct a series of experiments under fixed rollout budgets ranging from 5k to 50k. Table~\ref{tab:rollout-efficiency} summarizes the Pass@1 accuracy across various reasoning benchmarks for GRPO, One-Shot RL, and our proposed DoPR. It can be seen that under constrained rollout budgets, DoPR demonstrates a substantial performance advantage over GRPO. For instance, with only 10k total rollouts, DoPR achieves comparable or even superior Pass@1 accuracy relative to GRPO trained with 50k rollouts. This underscores the inefficiency of GRPO’s rollout strategy, which fails to distinguish high-impact samples from less informative ones during training. Consequently, a larger number of rollouts is required to reach similar levels of reasoning capability.  In contrast, DoPR leverages a dynamic selection mechanism to identify and focus on the most informative samples at each iteration, guided by a composite score that captures both reward uncertainty and exploration value. This utility-aware update strategy ensures that each rollout contributes more effectively to policy improvement, enabling faster convergence and stronger performance with significantly reduced computational cost. One-Shot RL also achieves competitive performance in extremely low-budget settings (e.g., 5k–10k rollouts). However, it fails to improve further as the rollout budget increases, due to its reliance on single-instance updates and lack of adaptive sample prioritization. In contrast, DoPR continues to scale effectively with larger budgets, eventually reaching performance parity with GRPO while maintaining superior efficiency throughout the training process. This reveals the core strength of DoPR: its ability to maximize the utility of limited rollout resources through principled and information-aware update scheduling. This makes it a practical and scalable solution for reinforcement learning in settings where rollout and computational efficiency are critical.

\begin{table*}[t]
	\caption{Ablation study of DoPR. We report Pass@1 after convergence and under a fixed rollout budget.}
	\label{tab:ablation}
	\centering
	\begin{tabular}{lccccccc}
		\toprule
		\rowcolor[HTML]{f8f9fa}
		\textbf{Method (Total Rollouts Fixed)} & \textbf{AIME24} & \textbf{AMC} & \textbf{MATH} & \textbf{MIN.} & \textbf{OLY.} & \textbf{GSM8K} & \textbf{Avg.} \\
		\midrule
		\rowcolor[HTML]{f8f9fa}
		\multicolumn{8}{l}{\textit{Final Performance}} \\
		DoPR        & 19.7 & 59.3 & 73.5 & 29.4 & 33.9 & 82.2 & 49.6 \\
		DoPR-UCB      & 19.3 & 59.0 & 73.6 & 29.1 & 33.9 & 81.9 & 49.5 \\
		DoPR-None   & 18.9 & 58.7 & 73.4 & 29.3 & 33.6 & 80.5 & 49.0 \\
		\rowcolor[HTML]{f8f9fa}
		\multicolumn{8}{l}{\textit{10k Rollout Budget}} \\
		DoPR        & 17.4 & 56.2 & 67.0 & 25.9 & 30.2 & 77.3 & 45.7 \\
		DoPR-UCB      & 16.7 & 55.0 & 66.9 & 25.4 & 29.5 & 76.9 & 45.1 \\
		DoPR-None   & 10.7 & 40.0 & 66.8 & 24.4 & 30.1 & 77.5 & 41.6 \\
		\bottomrule
	\end{tabular}
\end{table*}

\subsection{Ablation Studies}
We conduct an ablation study to investigate the effectiveness of the EM-UCB term. As shown in Table~\ref{tab:ablation}, all methods exhibit similar performance after converging, suggesting that the ultimate reasoning capability is primarily determined by the underlying reinforcement learning framework itself, and that the choice of individual training samples per update has limited influence on the final performance ceiling. In contrast, when the rollout budget is constrained, full DoPR substantially outperforms both ablated variants. This demonstrates the advantage of the proposed EM-UCB strategy in improving learning efficiency by directing rollouts toward more informative instances.

\begin{figure}[htb]
	\begin{minipage}[b]{.49\linewidth}
		\centering
		\centerline{\includegraphics[width=4.3cm]{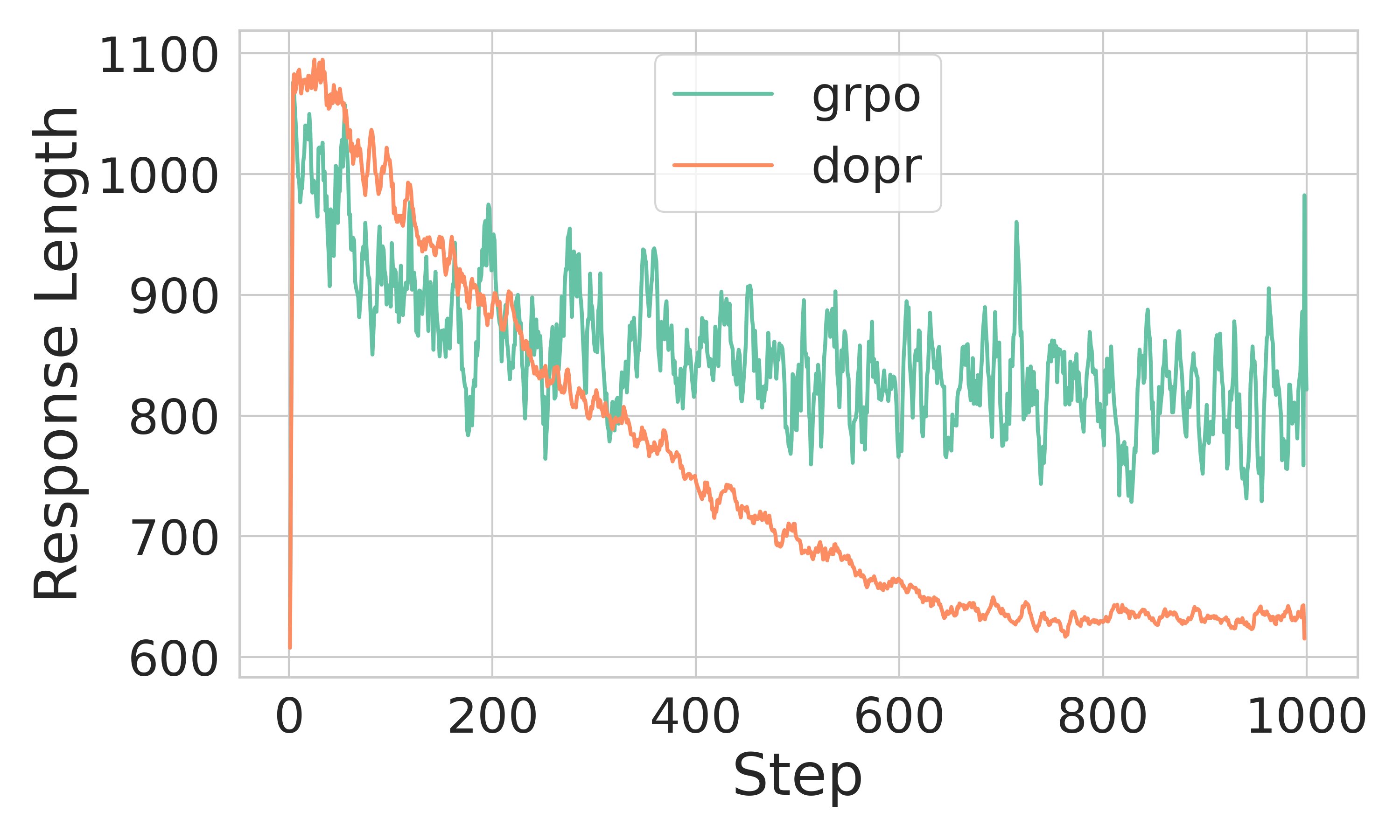}}
		\centerline{(a) Response length}\medskip
	\end{minipage}
	\begin{minipage}[b]{.49\linewidth}
		\centering
		\centerline{\includegraphics[width=4.3cm]{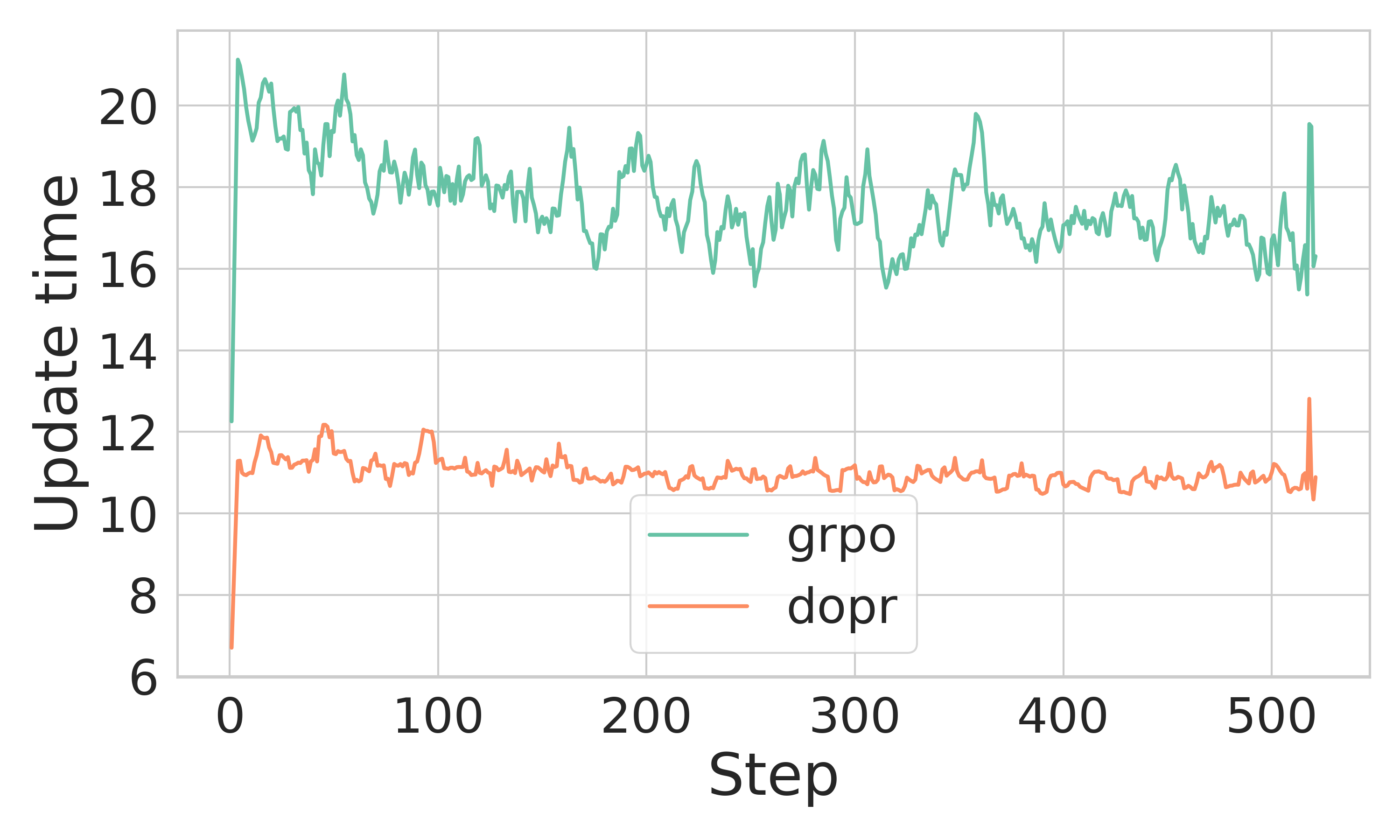}}
		\centerline{(b) Update time}\medskip
	\end{minipage}
	\caption{DoPR consistently yields shorter reasoning trajectories and faster update cycles, indicating improved runtime efficiency and a reduced computational burden for training.}
	\label{fig:efficiency}
\end{figure}

\subsection{Efficiency in Rollout and Update Time}
We further evaluate the computational efficiency of DoPR by analyzing two key metrics: the average response length and the per-step update time in training. As shown in Figure \ref{fig:efficiency}, GRPO maintains a relatively constant response length across training steps, reflecting its fixed group sampling strategy that does not adapt to the model's evolving confidence. In contrast, DoPR exhibits a rapid decline in response length over the course of training, eventually stabilizing at a significantly lower level. This indicates that our method progressively learns to generate more concise reasoning trajectories, which may due to the targeted sample selection and more confident policy updates. In terms of the update time, DoPR consistently achieves lower latency per training step compared to GRPO. This improvement stems from DoPR’s single-instance update mechanism, which requires only $G + (K-1)$ rollouts per step rather than the full $K \times G$ rollouts used by GRPO. The update time of DoPR remains stable throughout training, demonstrating its scalability and robustness under fixed compute budgets. Overall, these observations indicate that DoPR not only accelerates convergence but also improves runtime efficiency at both inference and training stages, making it a promising candidate for large-scale deployment in cost-sensitive environments.

\section{Conclusion}
In this work, we present a resource-efficient reinforcement learning framework for training reasoning-capable language models under RLVR. We begin by establishing a theoretical lower bound on the sample complexity of RLVR, providing the first formal understanding of data requirements during post-training. Empirically, we show that near-optimal reasoning performance can be achieved with surprisingly few training examples, challenging the prevailing assumption that large-scale datasets are indispensable for RL-based alignment. In parallel, we propose Dynamic One-Shot Policy Refinement (DoPR), a computation-efficient learning algorithm that significantly reduces rollout cost by dynamically selecting a single high-impact training sample per batch for policy updates. By combining reward volatility with exploration-aware scoring, DoPR achieves competitive superior reasoning performance across a suite of mathematical benchmarks, while dramatically reducing computational overhead. Together, our theoretical and algorithmic contributions offer a more sustainable path forward for scaling LLM reasoning training, both in terms of data requirement and compute efficiency.

\bibliography{acl_latex}

@article{Shao2024DeepSeekMathPT,
	title={DeepSeekMath: Pushing the Limits of Mathematical Reasoning in Open Language Models},
	author={Zhihong Shao and Peiyi Wang and Qihao Zhu and Runxin Xu and Jun-Mei Song and Mingchuan Zhang and Y. K. Li and Yu Wu and Daya Guo},
	journal={ArXiv},
	year={2024},
	volume={abs/2402.03300}
}

@article{DeepSeekAI2025DeepSeekR1IR,
	title={DeepSeek-R1: Incentivizing Reasoning Capability in LLMs via Reinforcement Learning},
	author={DeepSeek-AI and Daya Guo and Dejian Yang and Haowei Zhang and Jun-Mei Song and Ruoyu Zhang and Runxin Xu and Qihao Zhu and Shirong Ma and Peiyi Wang et al.},
	journal={ArXiv},
	year={2025},
	volume={abs/2501.12948}
}

@article{Liu2025UnderstandingRT,
	title={Understanding R1-Zero-Like Training: A Critical Perspective},
	author={Zi-Yan Liu and Changyu Chen and Wenjun Li and Penghui Qi and Tianyu Pang and Chao Du and Wee Sun Lee and Min Lin},
	journal={ArXiv},
	year={2025},
	volume={abs/2503.20783}
}

@article{Achiam2023GPT4TR,
  title={Gpt-4 Technical Report},
  author={Achiam, Josh and Adler, Steven and Agarwal, Sandhini and Ahmad, Lama and Akkaya, Ilge and Aleman, Florencia Leoni and Almeida, Diogo and Altenschmidt, Janko and Altman, Sam and Anadkat, Shyamal and others},
  journal={Arxiv},
  year={2023},
  volume={abs/2303.08774}
}

@article{Hu2025OpenReasonerZeroAO,
	title={Open-Reasoner-Zero: An Open Source Approach to Scaling Up Reinforcement Learning on the Base Model},
	author={Jingcheng Hu and Yinmin Zhang and Qi Han and Daxin Jiang and Xiangyu Zhang and Heung-yeung Shum},
	journal={ArXiv},
	year={2025},
	volume={abs/2503.24290}
}

@article{Contributors2024OpenAIOS,
	title={OpenAI o1 System Card},
	author={Foundational Contributors and Ahmed El-Kishky and Daniel Selsam and Francis Song and Giambattista Parascandolo and Hongyu Ren and Hunter Lightman and Hyung Won and Ilge Akkaya and Ilya Sutskever et al.},
	journal={ArXiv},
	year={2024},
	volume={abs/2412.16720}
}

@article{Reid2024Gemini1U,
	title={Gemini 1.5: Unlocking multimodal understanding across millions of tokens of context},
	author={Machel Reid and Nikolay Savinov and Denis Teplyashin and Dmitry Lepikhin and Timothy P. Lillicrap and Jean-Baptiste Alayrac and Radu Soricut and Angeliki Lazaridou and Orhan Firat and Julian Schrittwieser et al.},
	journal={ArXiv},
	year={2024},
	volume={abs/2403.05530}
}

@article{Zhang2025RightQI,
	title={Right Question is Already Half the Answer: Fully Unsupervised LLM Reasoning Incentivization},
	author={Qingyang Zhang and Haitao Wu and Changqing Zhang and Peilin Zhao and Yatao Bian},
	journal={ArXiv},
	year={2025},
	volume={abs/2504.05812}
}

@article{Lin2025CPPOAT,
	title={CPPO: Accelerating the Training of Group Relative Policy Optimization-Based Reasoning Models},
	author={Zhihang Lin and Mingbao Lin and Yuan Xie and Rongrong Ji},
	journal={ArXiv},
	year={2025},
	volume={abs/2503.22342}
}

@article{Xiong2025AMA,
	title={A Minimalist Approach to LLM Reasoning: from Rejection Sampling to Reinforce},
	author={Wei Xiong and Jiarui Yao and Yuhui Xu and Bo Pang and Lei Wang and Doyen Sahoo and Junnan Li and Nan Jiang and Tong Zhang and Caiming Xiong and Hanze Dong},
	journal={ArXiv},
	year={2025},
	volume={abs/2504.11343}
}

@inproceedings{Chen2024SelfPlayFC,
	author       = {Zixiang Chen and
	Yihe Deng and
	Huizhuo Yuan and
	Kaixuan Ji and
	Quanquan Gu},
	title        = {Self-Play Fine-Tuning Converts Weak Language Models to Strong Language
	Models},
	booktitle    = {International Conference on Machine Learning (ICML)},
	year         = {2024}
}

@inproceedings{Tajwar2024Preference,
	author       = {Fahim Tajwar and
	Anikait Singh and
	Archit Sharma and
	Rafael Rafailov and
	Jeff Schneider and
	Tengyang Xie and
	Stefano Ermon and
	Chelsea Finn and
	Aviral Kumar},
	title        = {Preference Fine-Tuning of LLMs Should Leverage Suboptimal, On-Policy
	Data},
	booktitle    = {International Conference on Machine Learning (ICML)},
	year         = {2024}
}

@inproceedings{Dong2023HowAI,
	title={How Abilities in Large Language Models are Affected by Supervised Fine-tuning Data Composition},
	author={Guanting Dong and Hongyi Yuan and Keming Lu and Chengpeng Li and Mingfeng Xue and Dayiheng Liu and Wei Wang and Zheng Yuan and Chang Zhou and Jingren Zhou},
	booktitle={Annual Meeting of the Association for Computational Linguistics (ACL)},
	year={2023}
}

@article{Gao2024OnDE,
	title={On Designing Effective RL Reward at Training Time for LLM Reasoning},
	author={Jiaxuan Gao and Shusheng Xu and Wenjie Ye and Weiling Liu and Chuyi He and Wei Fu and Zhiyu Mei and Guangju Wang and Yi Wu},
	journal={ArXiv},
	year={2024},
	volume={abs/2410.15115}
}

@article{Lambert2024TLU3P,
	title={T{\"U}LU 3: Pushing Frontiers in Open Language Model Post-Training},
	author={Nathan Lambert and Jacob Daniel Morrison and Valentina Pyatkin and Shengyi Huang and Hamish Ivison and Faeze Brahman and Lester James Validad Miranda and Alisa Liu and Nouha Dziri and Xinxi Lyu et al.},
	journal={ArXiv},
	year={2024},
	volume={abs/2411.15124}
}

@article{Team2025KimiKS,
	title={Kimi k1.5: Scaling Reinforcement Learning with LLMs},
	author={Kimi Team and Angang Du and Bofei Gao and Bowei Xing and Changjiu Jiang and Cheng Chen and Cheng Li and Chenjun Xiao and Chenzhuang Du and Chonghua Liao et al.},
	journal={ArXiv},
	year={2025},
	volume={abs/2501.12599}
}

@article{Wang2025BeyondT8,
	title={Beyond the 80/20 Rule: High-Entropy Minority Tokens Drive Effective Reinforcement Learning for LLM Reasoning},
	author={Shenzhi Wang and Le Yu and Chang Gao and Chujie Zheng and Shixuan Liu and Rui Lu and Kai Dang and Xionghui Chen and Jianxin Yang and Zhenru Zhang and Yuqiong Liu and An Yang and Andrew Zhao and Yang Yue and Shiji Song and Bowen Yu and Gao Huang and Junyang Lin},
	journal={ArXiv},
	year={2025},
	volume={abs/2506.01939}
}

@article{Wang2025ReinforcementLF,
	title={Reinforcement Learning for Reasoning in Large Language Models with One Training Example},
	author={Yiping Wang and Qing Yang and Zhiyuan Zeng and Liliang Ren and Lucas Liu and Baolin Peng and Hao Cheng and Xuehai He and Kuan Wang and Jianfeng Gao et al.},
	journal={ArXiv},
	year={2025},
	volume={abs/2504.20571}
}

@article{Chen2025SEEDGRPOSE,
	title={SEED-GRPO: Semantic Entropy Enhanced GRPO for Uncertainty-Aware Policy Optimization},
	author={Minghan Chen and Guikun Chen and Wenguan Wang and Yi Yang},
	journal={ArXiv},
	year={2025},
	volume={abs/2505.12346}
}

@article{Yue2025VAPOEA,
	title={VAPO: Efficient and Reliable Reinforcement Learning for Advanced Reasoning Tasks},
	author={Yu Yue and Yufeng Yuan and Qiying Yu and Xiaochen Zuo and Ruofei Zhu and Wenyuan Xu and Jiaze Chen and Chengyi Wang and Tiantian Fan and Zhengyin Du and Xiang Wei and Xiangyu Yu and Gaohong Liu and Juncai Liu and Lingjun Liu and Haibin Lin and Zhiqi Lin and Bole Ma and Chi Zhang and Mofan Zhang and Wang Zhang and Hang Zhu and Ru Zhang and Xin Liu and Mingxuan Wang and Yong-Xu Wu and Lin Yan},
	journal={ArXiv},
	year={2025},
	volume={abs/2504.05118}
}

@article{Xu2025NotAR,
	title={Not All Rollouts are Useful: Down-Sampling Rollouts in LLM Reinforcement Learning},
	author={Yixuan Even Xu and Yash Savani and Fei Fang and Zico Kolter},
	journal={ArXiv},
	year={2025},
	volume={abs/2504.13818}
}

@article{Cui2025TheEM,
	title={The Entropy Mechanism of Reinforcement Learning for Reasoning Language Models},
	author={Ganqu Cui and Yuchen Zhang and Jiacheng Chen and Lifan Yuan and Zhi Wang and Yuxin Zuo and Hao-Si Li and Yuchen Fan and Huayu Chen and Weize Chen and Zhiyuan Liu and Hao Peng and Lei Bai and Wanli Ouyang and Yu Cheng and Bowen Zhou and Ning Ding},
	journal={ArXiv},
	year={2025},
	volume={abs/2505.22617}
}

@inproceedings{KojimaGRMI22,
	author       = {Takeshi Kojima and
	Shixiang Shane Gu and
	Machel Reid and
	Yutaka Matsuo and
	Yusuke Iwasawa},
	editor       = {Sanmi Koyejo and
	S. Mohamed and
	A. Agarwal and
	Danielle Belgrave and
	K. Cho and
	A. Oh},
	title        = {Large Language Models are Zero-Shot Reasoners},
	booktitle    = {Advances in Neural Information Processing Systems (NeurIPS)},
	year         = {2022}
}

@inproceedings{ZhouSHWS0SCBLC23,
	author       = {Denny Zhou and
	Nathanael Sch{\"{a}}rli and
	Le Hou and
	Jason Wei and
	Nathan Scales and
	Xuezhi Wang and
	Dale Schuurmans and
	Claire Cui and
	Olivier Bousquet and
	Quoc V. Le and
	Ed H. Chi},
	title        = {Least-to-Most Prompting Enables Complex Reasoning in Large Language
	Models},
	booktitle    = {International Conference on Learning Representations (ICLR)},
	year         = {2023}
}

@inproceedings{Wei0SBIXCLZ22,
	author       = {Jason Wei and
	Xuezhi Wang and
	Dale Schuurmans and
	Maarten Bosma and
	Brian Ichter and
	Fei Xia and
	Ed H. Chi and
	Quoc V. Le and
	Denny Zhou},
	editor       = {Sanmi Koyejo and
	S. Mohamed and
	A. Agarwal and
	Danielle Belgrave and
	K. Cho and
	A. Oh},
	title        = {Chain-of-Thought Prompting Elicits Reasoning in Large Language Models},
	booktitle    = {Advances in Neural Information Processing Systems (NeurIPS)},
	year         = {2022}
}

@inproceedings{YaoYZS00N23,
	author       = {Shunyu Yao and
	Dian Yu and
	Jeffrey Zhao and
	Izhak Shafran and
	Tom Griffiths and
	Yuan Cao and
	Karthik Narasimhan},
	editor       = {Alice Oh and
	Tristan Naumann and
	Amir Globerson and
	Kate Saenko and
	Moritz Hardt and
	Sergey Levine},
	title        = {Tree of Thoughts: Deliberate Problem Solving with Large Language Models},
	booktitle    = {Advances in Neural Information Processing Systems (NeurIPS)},
	year         = {2023}
}

@inproceedings{WSLCNCZ23,
	author       = {Xuezhi Wang and
	Jason Wei and
	Dale Schuurmans and
	Quoc V. Le and
	Ed H. Chi and
	Sharan Narang and
	Aakanksha Chowdhery and
	Denny Zhou},
	title        = {Self-Consistency Improves Chain of Thought Reasoning in Language Models},
	booktitle    = {International Conference on Learning Representations (ICLR)},
	year         = {2023}
}

@inproceedings{Besta2023GraphOT,
	title={Graph of Thoughts: Solving Elaborate Problems with Large Language Models},
	author={Maciej Besta and Nils Blach and Ale Kubek and Robert Gerstenberger and Lukas Gianinazzi and Joanna Gajda and Tomasz Lehmann and Michal Podstawski and Hubert Niewiadomski and Piotr Nyczyk and Torsten Hoefler},
	booktitle={AAAI Conference on Artificial Intelligence (AAAI)},
	year={2023}
}

@article{Ye2025LIMOLI,
	title={LIMO: Less is More for Reasoning},
	author={Yixin Ye and Zhen Huang and Yang Xiao and Ethan Chern and Shijie Xia and Pengfei Liu},
	journal={ArXiv},
	year={2025},
	volume={abs/2502.03387}
}

@article{Uesato2022SolvingMW,
	title={Solving math word problems with process- and outcome-based feedback},
	author={Jonathan Uesato and Nate Kushman and Ramana Kumar and Francis Song and Noah Siegel and L. Wang and Antonia Creswell and Geoffrey Irving and Irina Higgins},
	journal={ArXiv},
	year={2022},
	volume={abs/2211.14275}
}

@inproceedings{ZhengC00WZL0LXZ23,
	author       = {Lianmin Zheng and
	Wei{-}Lin Chiang and
	Ying Sheng and
	Siyuan Zhuang and
	Zhanghao Wu and
	Yonghao Zhuang and
	Zi Lin and
	Zhuohan Li and
	Dacheng Li and
	Eric P. Xing et al.},
	title        = {Judging LLM-as-a-Judge with MT-Bench and Chatbot Arena},
	booktitle    = {Advances in Neural Information Processing Systems (NeurIPS)},
	year         = {2023}
}

@article{Bai2023QwenTR,
	title={Qwen Technical Report},
	author={Jinze Bai and Shuai Bai and Yunfei Chu and Zeyu Cui and Kai Dang and Xiaodong Deng and Yang Fan and Wenhang Ge and Yu Han and Fei Huang et al.},
	journal={ArXiv},
	year={2023},
	volume={abs/2309.16609}
}

@article{Yu2025DAPOAO,
	title={DAPO: An Open-Source LLM Reinforcement Learning System at Scale},
	author={Qiying Yu and Zheng Zhang and Ruofei Zhu and Yufeng Yuan and Xiaochen Zuo and Yu Yue and Tiantian Fan and Gaohong Liu and Lingjun Liu and Xin Liu et al.},
	journal={ArXiv},
	year={2025},
	volume={abs/2503.14476}
}

@article{Zhang2025SRPOAC,
	title={SRPO: A Cross-Domain Implementation of Large-Scale Reinforcement Learning on LLM},
	author={Xiaojiang Zhang and Jinghui Wang and Zifei Cheng and Wenhao Zhuang and Zheng Lin and Minglei Zhang and Shaojie Wang and Yinghan Cui and Chao Wang and Junyi Peng},
	journal={ArXiv},
	year={2025},
	volume={abs/2504.14286}
}

@article{Zeng2025SimpleRLZooIA,
	title={SimpleRL-Zoo: Investigating and Taming Zero Reinforcement Learning for Open Base Models in the Wild},
	author={Weihao Zeng and Yuzhen Huang and Qian Liu and Wei Liu and Keqing He and Zejun Ma and Junxian He},
	journal={ArXiv},
	year={2025},
	volume={abs/2503.18892}
}

@article{yang2024qwen25mathtechnicalreportmathematical,
	title={Qwen2.5-Math Technical Report: Toward Mathematical Expert Model via Self-Improvement}, 
	author={An Yang and Beichen Zhang and Binyuan Hui and Bofei Gao and Bowen Yu and Chengpeng Li and Dayiheng Liu and Jianhong Tu and Jingren Zhou and Junyang Lin and Keming Lu and Mingfeng Xue and Runji Lin and Tianyu Liu and Xingzhang Ren and Zhenru Zhang},
	journal={Arxiv},
	year={2024},
	volume={abs/2409.12122},
}

@inproceedings{Aitor2022Solving,
	author = {Lewkowycz, Aitor and Andreassen, Anders and Dohan, David and Dyer, Ethan and Michalewski, Henryk and Ramasesh, Vinay and Slone, Ambrose and Anil, Cem et al.},
	title = {Solving quantitative reasoning problems with language models},
	year = {2022},
	booktitle = {Advanced in Neural Information Processing Systems (NeurIPS)}
}

@inproceedings{Hendrycks2021MeasuringMP,
	title={Measuring Mathematical Problem Solving With the MATH Dataset},
	author={Dan Hendrycks and Collin Burns and Saurav Kadavath and Akul Arora and Steven Basart and Eric Tang and Dawn Xiaodong Song and Jacob Steinhardt},
	booktitle = {Advanced in Neural Information Processing Systems (NeurIPS)},
	year={2021}
}

@misc{deepscaler2025,
	title={DeepScaleR: Surpassing O1-Preview with a 1.5B Model by Scaling RL},
	author={Michael Luo and Sijun Tan and Justin Wong and Xiaoxiang Shi and William Y. Tang and Manan Roongta and Colin Cai and Jeffrey Luo and Li Erran Li and Raluca Ada Popa and Ion Stoica},
	year={2025},
	note={Notion Blog}
}

@inproceedings{Huang2024OlympicArenaBM,
	title={OlympicArena: Benchmarking Multi-discipline Cognitive Reasoning for Superintelligent AI},
	author={Zhen Huang and Zengzhi Wang and Shijie Xia and Xuefeng Li and Haoyang Zou and Ruijie Xu and Run-Ze Fan and Lyumanshan Ye and Ethan Chern and Yixin Ye et al.},
	booktitle = {Advanced in Neural Information Processing Systems (NeurIPS)},
	year={2024}
}

@article{cobbe2021gsm8k,
	title={Training Verifiers to Solve Math Word Problems},
	author={Cobbe, Karl and Kosaraju, Vineet and Bavarian, Mohammad and Chen, Mark and Jun, Heewoo and Kaiser, Lukasz and Plappert, Matthias and Tworek, Jerry and Hilton, Jacob and Nakano, Reiichiro and Hesse, Christopher and Schulman, John},
	journal={ArXiv},
	year={2021},
	volume={abs/2110.14168}
}

@article{Chen2025UnderstandingPA,
  title={Understanding Pre-training and Fine-tuning from Loss Landscape Perspectives},
  author={Huanran Chen and Yinpeng Dong and Zeming Wei and Yao Huang and Yichi Zhang and Hang Su and Jun Zhu},
  journal={ArXiv},
  year={2025},
  volume={abs/2505.17646},
}

@article{Aich2025FromST,
  title={From Sublinear to Linear: Fast Convergence in Deep Networks via Locally Polyak-Lojasiewicz Regions},
  author={Agnideep Aich and Ashit Baran Aich and Bruce Wade},
  journal={ArXiv},
  year={2025},
  volume={abs/2507.21429}
}

@inproceedings{peng2024navigating,
  title={Navigating the safety landscape: Measuring risks in finetuning large language models},
  author={Peng, Sheng Y and Chen, Pin-Yu and Hull, Matthew and Chau, Duen H},
  booktitle={Advances in Neural Information Processing Systems (NeurIPS)},
  year={2024}
}

@article{zhao2025ufo,
  title={UFO-RL: Uncertainty-Focused Optimization for Efficient Reinforcement Learning Data Selection},
  author={Zhao, Yang and Xiong, Kai and Ding, Xiao and Du, Li and Sun, Zhouhao and Guan, Jiannan and Zhang, Wenbin and Liu, Bin and Hu, Dong and Qin, Bing and others},
  journal={Arxiv},
  year={2025},
  volume={abs/2505.12457}
}

@article{kong2025rethinking,
  title={Rethinking the Sampling Criteria in Reinforcement Learning for LLM Reasoning: A Competence-Difficulty Alignment Perspective},
  author={Kong, Deyang and Guo, Qi and Xi, Xiangyu and Wang, Wei and Wang, Jingang and Cai, Xunliang and Zhang, Shikun and Ye, Wei},
  journal={Arxiv},
  volume={abs/2505.17652},
  year={2025}
}

@article{kaelbling1994associative,
  title={Associative reinforcement learning: A generate and test algorithm},
  author={Kaelbling, Leslie Pack},
  journal={Machine Learning},
  volume={15},
  number={3},
  pages={299--319},
  year={1994}
}

@inproceedings{garivier2011upper,
  title={On upper-confidence bound policies for switching bandit problems},
  author={Garivier, Aur{\'e}lien and Moulines, Eric},
  booktitle={International conference on algorithmic learning theory},
  pages={174--188},
  year={2011}
}

\appendix

\section{Detailed Proof of Theorem 1}
For a pre-training policy $\pi_{\theta_0}$, its performance gap with the optimal policy $\pi_{\theta^{\star}}$ on a reasoning task is:
\begin{equation}
    \mathbb{E}_{x\sim\mathcal{D}}[P^{\pi_{\theta^{\star}}}(x)-P^{\pi_{\theta_0}}(x)]=\epsilon,
\end{equation}
where $x$ is an instance sampled from the dataset $\mathcal{D}$, and $P^{\pi}(x)$ denotes the expected verification reward.

\begin{theorem}
Consider an optimization procedure where the policy is updated using a single sample at each step, let $\pi_{\theta^\star}$ denotes the optimal policy and $\pi_{\theta_N}$ denotes the policy after $N$ updates. To guarantee that the expected performance gap satisfies
\begin{equation}
    \mathbb{E}_{x \sim \mathcal{D}} \left[ P^{\pi_{\theta^\star}}(x) - P^{\pi_{\theta_N}}(x) \right] \leq \epsilon',
\end{equation}
it suffices that the number of steps $N$ satisfies
\begin{equation}
    N \geq \mathcal{O}( \ln \frac{\epsilon}{\epsilon'}),
\end{equation}
where $\epsilon$ denotes the initial performance gap.
\end{theorem}

\begin{proof}
Let \( J(\theta) = \mathbb{E}_{x \sim \mathcal{D}} \left[ P^{\pi_\theta}(x) \right] \) denote the expected performance of the policy parameterized by \( \theta \). Assume that \( J \) is locally \( L \)-smooth, i.e., there exists a constant $L>0$ and a neighborhood $\mathcal{N}$ containing the initial policy $\theta_0$, such that the following Lipschitz inequality holds for all \( \theta, \theta' \in \mathcal{N} \):
\begin{equation}
    J(\theta') \geq J(\theta) + \langle \nabla J(\theta), \theta' - \theta \rangle - \frac{L}{2} \| \theta' - \theta \|^2.
\end{equation}
By the policy gradient theorem, \( J \)'s gradient is given by:
\begin{equation}
    \nabla J(\theta) = \mathbb{E}_{x \sim \mathcal{D}} \left[ \nabla_{\theta} P^{\pi_\theta}(x) \right].
\end{equation}
Assume that the policy is updated at each step via gradient ascent with a fixed learning rate \( \alpha > 0 \):
\begin{equation}
    \theta_{t+1} = \theta_t + \alpha g_t,
\end{equation}
where \( g_t \) is an unbiased stochastic estimate of the true gradient, i.e., \( \mathbb{E}[g_t] = \nabla J(\theta_t) \), and the variance of \( g_t \) is bounded by $Var(g_t)\leq\delta^2$.
Due to the $L$-smoothing, 
\begin{equation}
	\begin{aligned}
		J(\theta_{(t+1)}&\geq J(\theta_t)+\langle \nabla J(\theta_t),\alpha g_t\rangle-\frac{L}{2}\| \alpha g_t\|^2 \\
		&=J(\theta_t)+\alpha\langle \nabla J(\theta_t),g_t\rangle-\frac{L\alpha^2}{2}\| g_t\|^2.
	\end{aligned}
\end{equation}
Taking expectations $\mathbb{E}\left[\cdot|\theta_t\right]$, we obtain:
\begin{equation}
\small
	\begin{aligned}
		&\mathbb{E}\left[J(\theta_{t+1}|\theta_t) \right]\geq J(\theta_t)+\alpha\|\nabla J(\theta_t)\|^2-\frac{L\alpha^2}{2}\mathbb{E}\left[\|g_t\|^2|\theta_t\right]\\
		&\geq J(\theta_t)+\alpha\|\nabla J(\theta_t)\|^2-\frac{L\alpha^2}{2}(\|\nabla J(\theta_t)\|^2+\delta^2)\\
		&=J(\theta_t)+\alpha(1-\frac{L\alpha}{2})||\nabla J(\theta_t)||^2-\frac{L\alpha^2\delta^2}{2},
	\end{aligned}
\end{equation}
Pretraining and SFT endow the model with an initial reasoning capability. Coupled with the explicit gradient signal provided by RLVR, the objective \( J \) satisfies a local Polyak–Łojasiewicz (PL) condition in the neighborhood induced by pretraining and SFT. Intuitively, the combination of pretraining/SFT (which shapes a favorable local manifold) and the supervised RLVR gradients renders the landscape locally well-conditioned for efficient optimization. That is, \( J \) satisfies
\begin{equation}
    \| \nabla J(\theta_t) \|^2 \geq c \left( J(\theta^\star) - J(\theta_t) \right) = c \Delta_t,
\end{equation}
for some constant \( c > 0 \), where \( \theta^\star \) is the optimal parameter and \( \Delta_t = J(\theta^\star) - J(\theta_t) \) denotes the performance gap at step \( t \). Then, we have:
\begin{equation}
\small
    \mathbb{E} \left[ J(\theta_{t+1}) - J(\theta_t) \mid \theta_t \right] \geq \alpha c (1 - \frac{L\alpha}{2}) \Delta_t - \frac{L\alpha^2 \delta^2}{2}.
\end{equation}
Let \( \eta = \alpha c / 2 \), and choose \( \alpha \) such that \( 1 - \frac{L\alpha}{2} \geq \frac{1}{2} \), we further require
\[
    \frac{L \alpha^2 \delta^2}{2} \leq \frac{\eta \epsilon'}{2},
\]
which leads to the recursive inequality:
\begin{equation}
	\begin{aligned}
		&\mathbb{E} \left[ \Delta_{t+1} \mid \theta_t \right] \\
		&=\Delta_t-\mathbb{E}\left[J(\theta_{t+1})-J(\theta_t)\mid \theta_t\right] \\
		&\leq (1 - \eta) \Delta_t + \frac{\epsilon'}{2}.
	\end{aligned}
\end{equation}
Let $\eta=\frac{\alpha c}{2}$ and $\frac{L\alpha^2\delta^2}{2}\leq\frac{\eta\epsilon'}{2}$, we can get
\begin{equation}
	\mathbb{E}\left[\Delta_{t+1}\right]\leq(1-\eta)\mathbb{E}\left[\Delta_t\right]+\frac{\eta\epsilon'}{2},
\end{equation}
If $\Delta_t\geq\epsilon'$, we obtain:
\begin{equation}
    \mathbb{E} \left[ \Delta_t \right] \leq (1 - \eta)^t \Delta_0 + \frac{\epsilon'}{2},
\end{equation}
where \( \Delta_0 = J(\theta^\star) - J(\theta_0) = \epsilon \) is the initial gap. To ensure \( \mathbb{E} \left[ \Delta_t \right] \leq \epsilon' \), it suffices that:
\begin{equation}
    (1 - \eta)^t \epsilon \leq \frac{\epsilon'}{2}.
\end{equation}
Taking logarithms on both sides yields:
\begin{equation}
    t \geq \frac{\ln(\epsilon'/2) - \ln(\epsilon)}{\ln(1 - \eta)} = \mathcal{O}( \ln \frac{\epsilon}{\epsilon'}),
\end{equation}
\end{proof}

\end{document}